\documentclass{article}

\usepackage{microtype}
\usepackage{graphicx}
\usepackage{subfigure}
\usepackage{booktabs} 
\usepackage{mdframed}

\usepackage{hyperref}

\usepackage[accepted]{icml2023}

\usepackage{amsmath}
\usepackage{amssymb}
\usepackage{mathtools}
\usepackage{amsthm}

\usepackage[capitalize,noabbrev]{cleveref}

\theoremstyle{plain}
\newtheorem{theorem}{Theorem}[section]

\newtheorem{lemma}[theorem]{Lemma}

\theoremstyle{definition}

\theoremstyle{remark}

\usepackage[textsize=tiny]{todonotes}

\definecolor{calpolypomonagreen}{rgb}{0.12, 0.7, 0.17}

\mdfdefinestyle{MyFrame}{
	linecolor=black,
	backgroundcolor=gray!10,
	innertopmargin=5pt,
	innerbottommargin=10pt,
	innerrightmargin=5pt,
	innerleftmargin=5pt,
}
\newcommand{\pitfallenv}[4]{%
	\vspace{2mm}
	\noindent\begin{minipage}{\columnwidth}%
	\begin{mdframed}[style=MyFrame]
		\textbf{#1.} #2
	\end{mdframed}%
\end{minipage}%
}

\icmltitlerunning{Theoretically Principled Trade-off for Stateful Defenses against Query-Based Black-Box Attacks}

\begin{document}

\twocolumn[
\icmltitle{Theoretically Principled Trade-off for Stateful Defenses against Query-Based Black-Box Attacks}



\icmlsetsymbol{equal}{*}

\begin{icmlauthorlist}
\icmlauthor{Ashish Hooda}{equal,1}
\icmlauthor{Neal Mangaokar}{equal,2}
\icmlauthor{Ryan Feng}{2}
\icmlauthor{Kassem Fawaz}{1}
\icmlauthor{Somesh Jha}{1}
\icmlauthor{Atul Prakash}{2}
\end{icmlauthorlist}

\icmlaffiliation{1}{University of Wisconsin-Madison}
\icmlaffiliation{2}{University of Michigan}

\icmlcorrespondingauthor{Ashish Hooda}{ahooda@wisc.edu}
\icmlcorrespondingauthor{Neal Mangaokar}{nealmgkr@umich.edu}

\icmlkeywords{Machine Learning, ICML}

\vskip 0.3in
]



\printAffiliationsAndNotice{\icmlEqualContribution} 

\begin{abstract}
 Adversarial examples threaten the integrity of machine learning systems with alarming success rates even under constrained black-box conditions. Stateful defenses have emerged as an effective countermeasure, detecting potential attacks by maintaining a buffer of recent queries and detecting new queries that are too similar. However, these defenses fundamentally pose a trade-off between attack detection and false positive rates, and this trade-off is typically optimized by hand-picking feature extractors and similarity thresholds that empirically work well. There is little current understanding as to the formal limits of this trade-off and the exact properties of the feature extractors/underlying problem domain that influence it. This work aims to address this gap by offering a theoretical characterization of the trade-off between detection and false positive rates for stateful defenses. We provide upper bounds for detection rates of a general class of feature extractors and analyze the impact of this trade-off on the convergence of black-box attacks. We then support our theoretical findings with empirical evaluations across multiple datasets and stateful defenses.
\end{abstract}
\section{Introduction}



Adversarial examples pose a significant threat to the security and integrity of machine learning systems~\cite{8578273,Sayles_2021_CVPR}. These examples are subtly manipulated inputs that deceive the models and cause misclassifications~\cite{DBLP:journals/corr/SzegedyZSBEGF13,7958570,hooda2022adversarially}. Even in the challenging black-box setting, where the attacker has limited information access, adversarial examples have been remarkably successful~\cite{ilyas2018black,chen2020hopskipjumpattack,feng2022graphite,maho2021surfree,andriushchenko2020square,li2020qeba}.

Recent research has shown that stateful defenses offer a promising approach to mitigate the impact of such attacks~\cite{li2022blacklight, choi2023piha, chen2020stateful}. These defenses leverage the observation that black-box attackers often submit numerous highly similar queries, e.g., querying nearby points for gradient estimation. To counter this, stateful defenses maintain a buffer of recent queries and compare incoming queries in some feature space to identify potential attacks. If the similarity between queries exceeds a predefined threshold, defensive action is taken, e.g., banning the user's account~\cite{chen2020stateful} or rejecting queries~\cite{li2022blacklight}. 

The success of stateful defenses hinges on their ability to detect and flag attack queries without flagging benign ones. 
This suggests the existence of a trade-off between the detection and false positive rates of a stateful defense (much like the trade-off between robustness and accuracy for existing white-box defenses~\cite{tsipras2018robustness,yang2020closer,raghunathan2020understanding}).
In light of this, existing defenses typically tune their similarity threshold to manipulate the trade-off, i.e., such that the defense only permits an empirically computed false positive rate. However, this does not provide any guarantees for the detection rate, and little is currently known about the exact properties of the feature spaces and problem domains that influence this trade-off. This work aims to address this gap by theoretically characterizing the trade-off between the detection rate and the false positive rate of stateful defenses. Specifically, we provide upper bounds for the detection rate for a general class of feature extractors. We then empirically validate that the takeaways from these bounds hold for multiple datasets and defenses and also analyze how this trade-off affects the convergence of black-box attacks.
\section{Background}

\subsection{Black-box Attacks}
Adversarial Examples are perturbed inputs that intentionally mislead or deceive machine learning models. Specifically, given an image $\mathbf{x}$ with label $y$ and a classifier $f$, such attacks aim to construct an adversarial example $\textbf{x}_{adv}$ such that:
\begin{equation}
    f(\mathbf{x}_{adv}) \neq y ~~\text{and}~~ ||\mathbf{x}_{adv} - \mathbf{x}||_p \leq \epsilon 
\end{equation}
where $\epsilon$ is the perturbation budget per some $\ell_p$ norm. In the black-box setting, these attacks only have on query access to the model. One common characteristic of black-box attacks is the use of similar queries to gather information about the model's behavior. Specifically, by making queries with slight perturbations to the input and observing the corresponding model outputs, attackers can gain insights into the model's decision-making process. 

Consider the initial stage of many black-box adversarial attacks, which involves estimating the direction to move the input to achieve the desired adversarial effect. For example, the NES~\cite{ilyas2018black}, HSJA~\cite{chen2020hopskipjumpattack}, and QEBA~\cite{li2020qeba} attacks estimate the gradient by sampling nearby points from a Gaussian (or similar) probability distributions, and computing finite differences over these points. Other attacks such as SurFree~\cite{maho2021surfree} and Square~\cite{andriushchenko2020square} also sample nearby points to estimate a ``random search'' direction (not a gradient) in which to move the input. We will often refer to the interplay between such queries made during the direction estimation stage and a stateful defense, particularly because the attack's overall convergence properties are often directly influenced by choice of direction. 


\subsection{Stateful Defenses}
The overall intuition behind stateful defenses is that black-box attackers often submit highly similar queries as part of the optimization procedure for their chosen adversarial task. These highly similar queries can then be detected. Defenses such as Blacklight~\cite{li2022blacklight} have reduced attack success rate (ASR) of state-of-the-art black-box attacks to as low as 0\%.

A stateful defense typically comprises a classifier $f$, feature extractor $H$ (with some associated distance metric), query store $q$, and threshold $\tau$. The defense then compares an incoming query against all queries stored in $q$. If similarity with any example in $q$ exceeds $\tau$, the defense deploys preventive measures such as query rejection or account banning.

Different stateful defenses primarily vary in their choices of $H$. Specifically, some defenses such as Blacklight and PIHA~\cite{li2022blacklight, choi2023piha} leverage discrete-valued metrics such as hamming distance over hashes, e.g., SHA-256 hashes of quantized pixels. Others, such as Stateful Defense (SD)~\cite{chen2020stateful}, employ real-valued metrics, e.g., $\ell_2$ distance between embeddings from neural similarity encoders. In this work, we evaluate  Blacklight and PIHA since they are available for both the CIFAR-10 and ImageNet datasets.

\textbf{Model Stealing}
Recent work has also proposed stateful defenses against model-stealing attacks. Such attacks aim to steal a local ``clone'' model $f^{c}$ such that the behavior of $f^{c}$ is similar to that of $f$. Defenses such as SEAT~\cite{zhang2021seat} have also been successful here and can force the attacker to create as many as 65 accounts to steal a single model. This success can be similarly explained by the submission of highly similar queries. For example, at iteration $t$ of a Jacobian-based Augmentation (JBA) attack~\cite{papernot2017practical}, the adversary constructs a ``useful'' but highly similar query $\mathbf{x}_{t+1}$ by perturbing previous query $\mathbf{x}_{t}$ so that it maximizes the loss $\mathcal{L}$ of $f^{c}_t$:
\begin{equation}
    \mathbf{x}_{t+1} = \mathbf{x}_t + \eta * sgn(\nabla_{\textbf{x}_t} \mathcal{L}(\mathbf{x}_t, f(\mathbf{x}_t)))
\end{equation}
where $\eta$ is some step size.






\section{Trade-offs between Detection and False Positives}\label{sec:tradeoff}

In this section, we demonstrate that there exists an implicit trade-off between detecting attack queries and avoiding false positives in the context of stateful defenses. We begin with a constructive model through which we provide explicit characterizations of the feature extractor and data distributions. We use this toy model to highlight the trade-off, and then relax the assumptions to provide a more general bound that highlights the direct influence of the feature extractor and the problem domain.

\subsection{Toy Model}


\noindent\textbf{Feature extractor.} We begin by considering an explicit class of feature extractors based on simple quantization. The feature extractor is given by $H : \mathbb{R}^d \rightarrow \mathbb{Z}^d$ with a discrete output space. Specifically, 
\begin{equation}\label{eqn:hash_fn_toy_model}
    H(\mathbf{x}) = \lfloor \mathbf{x} + \mathbf{0.5}\rfloor
\end{equation}
where the $\lfloor . \rfloor$ operation is element-wise. Many defenses employ quantization to provide perceptual similarity~\cite{li2022blacklight,choi2023piha}. In this model, we consider a query to be an attack query if and only if it produces the exact same features as that of a prior query. Later, in Section~\ref{sec:general_analysis} we expand beyond the toy model to consider the case where $H$ is a generic feature extractor, and queries are considered attack queries when their features are within some distance $\tau$ of a prior query.

\begin{figure}
    \centering
    \includegraphics[width=0.5\textwidth]{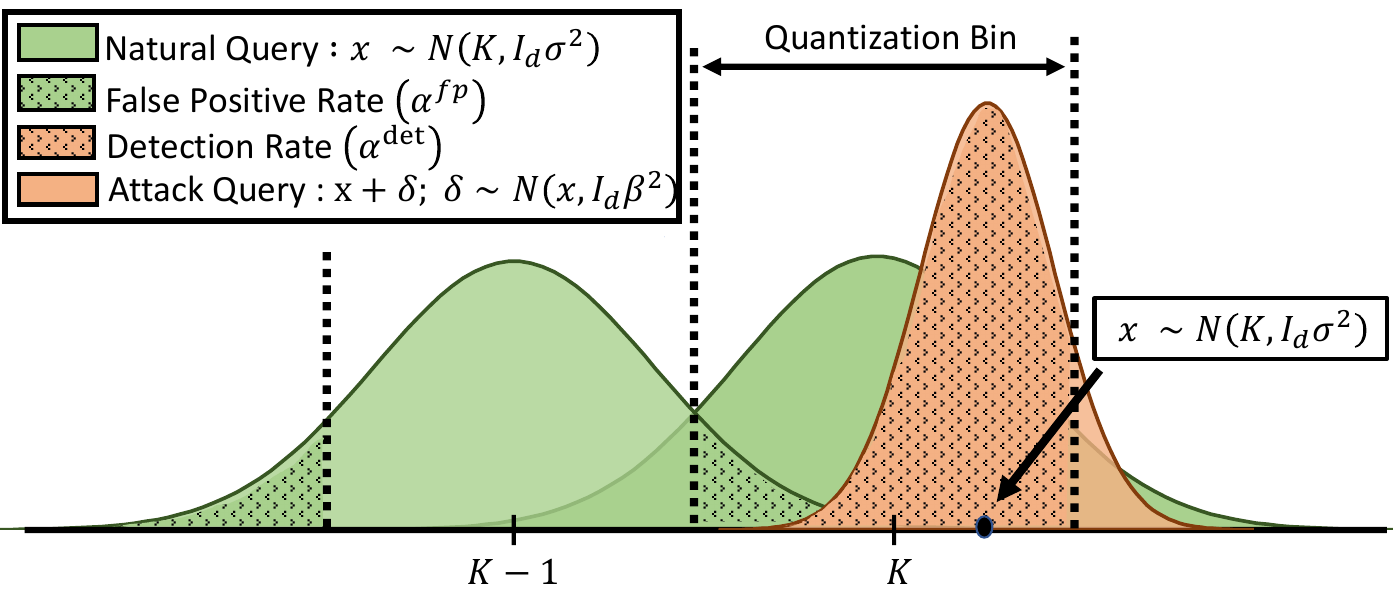}
    \caption{Illustration of the Toy Model in 1-D. We assume that any natural query $\mathbf{x}$ is sampled from a distinct Gaussian distribution (green). Two such distributions are shown, centered at $K-1$ and $K.$ For a given natural query $\mathbf{x}$, the attack queries $\mathbf{x}+\boldsymbol{\delta}$ are sampled from another Gaussian distribution (orange) centered around a natural query. The feature extractor is designed to map each natural query to a unique output. Therefore, $H$ maps all values within each quantization bin to the same output. This means that the green shaded area represents $\alpha^{fp}$, and the orange shaded area represents $\alpha^{det}$ for the attack queries. }
    \label{fig:toy_model}
\end{figure}

\noindent\textbf{Natural Query Distribution.} Stateful defenses assume that natural images are sufficiently ``spread out'', or dissimilar enough such that they can be distinguished. Therefore, for our model we assume that natural images originate from one of several Gaussian distributions, which are uniformly dispersed across input space $\mathbb{R}^d$ \footnote{For the case where the input space is constrained, for instance to [0,255], the natural images can instead be sampled from truncated Gaussian distributions.}. Each natural image is obtained from a distinct Gaussian distribution. This may be viewed as a ``best case'' situation for the defense, where natural images are sufficiently spread out across the input space to avoid false positives. For simplicity, we assume isotropic Gaussian distributions: $\mathcal{N}(\mathbf{p}, \mathbf{I_d} \sigma^2)$ where $\mathbf{p} \in \mathbb{Z}^d$. Intuitively, when applying $H$ to a natural image $\mathbf{x} \sim \mathcal{N}(\mathbf{p}, \mathbf{I_d} \sigma^2)$, it should output the discrete feature vector $\mathbf{p}$ with high probability.

\noindent\textbf{Attack Query Distribution.} To estimate the gradient at input $\mathbf{x}$, a Monte Carlo simulation approach would require sampling a total of $q$ perturbations $\{\mathbf{x},\mathbf{x + \boldsymbol{\delta}_1},...,\mathbf{x + \boldsymbol{\delta}_q}\}$. For our model, we consider the distribution of perturbations for $\mathbf{x}$ to be $\mathcal{N}(0, \mathbf{I_d}\beta^2)$, i.e., the adversary is estimating a gradient using finite differences on a Gaussian basis~\cite{ilyas2018black}.

Given the setting described above (also illustrated in Figure~\ref{fig:toy_model}), we now present the following result, which bounds the detection rate with the false positive rate:

\begin{theorem}\label{theorem:toy}
{Let the adversary sample a natural image $\mathbf{x}$ from one of the above distributions $\mathcal{N}(\mathbf{p},\mathbf{I_d}\sigma^2)$, and perturb it with $\boldsymbol{\delta} \sim \mathcal{N}(0, \mathbf{I_d}\beta^2)$ to estimate a gradient. Given that the stateful defense incurs a false positive rate $\alpha^{fp}$, the detection rate $\alpha^{det}$ for the perturbed query $\mathbf{x} + \boldsymbol{\delta}$ is then bounded as follows:}
\begin{equation}
    \alpha^{det} \leq 1 - \left(2 - 2\Phi\left(0.5\beta^{-1}\right) \right)^d(1-\alpha^{fp})
\end{equation}
\end{theorem}
\begin{proof}[Proof]
$H$ fails to detect the attack query $\mathbf{x + \boldsymbol{\delta}}$ if and only if $H(\mathbf{x} + \boldsymbol{\delta}) \neq H(\mathbf{x})$. Therefore,
\begin{align}
    \alpha^{det} &= 1 - \mathbb{P}[H(\mathbf{x}) \neq H(\mathbf{x+\boldsymbol{\delta}})]\\
    & \leq 1 - \mathbb{P}[H(\mathbf{x+\boldsymbol{\delta}}) \neq \mathbf{p},H(\mathbf{x}) = \mathbf{p}]\\
    &= 1 - \mathbb{P}[H(\mathbf{x+\boldsymbol{\delta}}) \neq \mathbf{p}~|~H(\mathbf{x}) = \mathbf{p}]\mathbb{P}[H(\mathbf{x}) = \mathbf{p}]\\
    &\leq 1 - \mathbb{P}[H\left(\mathbf{p}+\mathbf{\boldsymbol{\delta}}\right) \neq \mathbf{p}]\mathbb{P}[H(\mathbf{x}) = \mathbf{p}]\\
    &= 1 - \mathbb{P}[||\mathbf{\boldsymbol{\delta}}||_{\infty} > 0.5]\mathbb{P}[H(\mathbf{x}) = \mathbf{p}]\\
    &= 1 - \left(2 - 2\Phi\left(0.5\beta^{-1}\right) \right)^d(1-\alpha^{fp})
\end{align}
where $\Phi$ is the cummulative distribution function of $\mathcal{N}(0,1)$. Note that to go from (7) to the inequality in (8), we assign a specific value $\mathbf{x} = \mathbf{p}$, i.e., placing $\mathbf{x}$ at the center of the quantization bin for $H$ (see Equation~\ref{eqn:hash_fn_toy_model}). By placing it at the center, the probability of evasion when adding $\boldsymbol{\delta}$ is minimized, and the resulting event is also independent of event $H(\mathbf{x})=\mathbf{p}$. Finally, going from (9) to (10) uses standard results for the CDF of a multivariate Gaussian.  
\end{proof}




\pitfallenv{Takeaway}
{\label{takeaway_toy_model} There exists a trade-off between the detection rate $\alpha^{det}$ and the false positive rate $\alpha^{fp}$, i.e., decreasing $\alpha^{fp}$ also decreases the upper bound for $\alpha^{det}$. Furthermore, this trade-off also depends on the standard deviation $\beta$ of the perturbation distribution, i.e., high values of $\beta$ lead to a lower detection rate.}

\subsection{General Analysis}\label{sec:general_analysis}
Recall that our toy model assumed a quantization-based feature extractor and a uniform natural image distribution. We now extend our results to a more generic perceptual feature extractor and image distribution. Specifically, consider $H : \mathbb{R}^d \rightarrow \mathbb{R}^y$ where $y$ is the dimensionality of the output feature space. We assume $H$ to be Lipschitz continuous with constants $K_L$ and $K_U$ :
\begin{align}\label{eq:lipshitz}
    K_L ||\mathbf{x_1-x_2}|| \leq ||H(\mathbf{x_1})-H(\mathbf{x_2})|| \leq K_U ||\mathbf{x_1-x_2}||,
\end{align}
$\forall~(\mathbf{x_1,x_2}) \in \mathbb{R}^d$. Note that we no longer assume the implementation of $H$ as in the toy model; the continuity assumption here is only needed to ensure that $H$ captures perceptual similarity, i.e., similar images should indeed have similar features. Furthermore, since $H$ is now continuous, we extend to a threshold based detection setting i.e. a query $\mathbf{x}$ is considered an attack query if and only if $||H(\mathbf{x})-H(\mathbf{x_h})|| \leq \tau$ where $\mathbf{x_h}$ is any historical query. Given these changes, we can now re-analyze the detection $\alpha^{det}$ for a perturbed query $\mathbf{x} + \boldsymbol{\delta}$:

\begin{theorem}\label{theorem:general}
{Let the adversary sample natural image $\mathbf{x}$, and perturb it with $\boldsymbol{\delta} \sim \mathcal{N}(0, \mathbf{I_d}\beta^2)$ to estimate a gradient. For a false positive rate $\alpha^{fp}$, the detection rate $\alpha^{det}$ for perturbed query $\mathbf{x} + \boldsymbol{\delta}$ is then bounded as follows:}
\begin{align}
    \alpha^{det} \leq \frac{1}{\Gamma(\frac{d}{2})} \gamma\left(\frac{d}{2},\frac{1}{2}\left(\frac{K_U}{K_L}\frac{M_\mathcal{D}}{\beta}\frac{1}{1-\alpha^{fp}}\right)^2\right)
\end{align}
where $M_{\mathcal{D}} = {\mathbb{E}}[||\mathbf{x_1}-\mathbf{x_2}|| ]$, i.e., the expected spread of natural queries, and $\gamma$ and $\Gamma$ are the monotonic lower incomplete and complete Gamma functions respectively.
\end{theorem}

\begin{proof}[Proof] $H$ fails to detect the attack query $\mathbf{x + \boldsymbol{\delta}}$ if and only if $||H(\mathbf{x}) - H(\mathbf{x+\boldsymbol{\delta}})|| > \tau$. Therefore,
\begin{align}\label{eq:general_alpha_det}
    \alpha^{det} = {\mathbb{P}}\left[ ||H(\mathbf{x}) - H(\mathbf{x+\boldsymbol{\delta}})|| \leq \tau \right]
\end{align}
Similarly, $H$ produces a false positive for two natural images $\mathbf{x_1}$ and $\mathbf{x_2}$ if and only if $||H(\mathbf{x_1}) - H(\mathbf{x_2})|| \leq \tau$. Therefore,
\begin{align}\label{eq:general_alpha_fp}
    \alpha^{fp} = {\mathbb{P}}\left[ ||H(\mathbf{x_1}) - H(\mathbf{x_2})|| \leq \tau \right]
\end{align}
Using Equation~\ref{eq:lipshitz} with~\ref{eq:general_alpha_det} and~\ref{eq:general_alpha_fp}:
\begin{align}\label{eq:det}
    \alpha^{det} \leq {\mathbb{P}}\left[ ||\mathbf{\boldsymbol{\delta}}|| \leq \frac{\tau}{K_L} \right]
\end{align}
\begin{align}\label{eq:fp}
    \alpha^{fp} \geq {\mathbb{P}}\left[ ||\mathbf{x_1} - \mathbf{x_2}|| \leq \frac{\tau}{K_U} \right]
\end{align}


Finally, using a CDF for the norm of a Gaussian, i.e., a chi-distribution in Equation~\ref{eq:det} and Markov's inequality in Equation~\ref{eq:fp}, we get:
\begin{align}
    \alpha^{det} \leq \frac{1}{\Gamma(\frac{d}{2})} \gamma\left(\frac{d}{2},\frac{1}{2}\left(\frac{K_U}{K_L}\frac{M_\mathcal{D}}{\beta}\frac{1}{1-\alpha^{fp}}\right)^2\right)
\end{align}
where: 
\begin{align}
    \gamma(s,x) = \int_0^x t^{s-1}e^{-t}dt\\
    \Gamma(s) = \int_0^\infty t^{s-1}e^{-t}dt
\end{align}
\end{proof}

\pitfallenv{Takeaway}
{\label{takeaway_general} The trade-off observed in the toy model also extends to the more general setting, i.e., the detection rate $\alpha^{det}$ and the false positive rate $\alpha^{fp}$ are still at odds with each other. Furthermore, this trade-off depends upon the standard deviation $\beta$ of the perturbation distribution, the expected spread $M_{\mathcal{D}}$ of natural queries, and the Lipschitz constant ratio $K_U/K_L$ of $H$.}

\section{Experiments}
Motivated by our analysis in Section~\ref{sec:tradeoff}, we conduct experiments to validate our findings empirically, and thus answer the following questions:

\noindent\textbf{Q1. How does the trade-off empirically depend upon the spread, i.e., variance $\beta$ of the attack queries?}

\noindent\textbf{Q2. How does the trade-off empirically depend upon the Lipschitz constant ratio $K_U / K_L$ of the feature extractor?}

\noindent\textbf{Q3. What are the implications of the trade-off for the convergence of black-box attacks?}

\subsection{Experimental Setup}
\noindent\textbf{Feature extractors.} We focus our evaluation on feature extractors from two state-of-the-art stateful defenses: Blacklight~\cite{li2022blacklight} and PIHA~\cite{choi2023piha}. Below we provide detailed descriptions and hyper-parameters for both.

Blacklight operates on an input image with pixel values in the range of [0, 255]. First, it discretizes the pixels into bins of size 50. Second, a sliding window technique is applied to the discretized image, utilizing a window size of 20 for TinyImages~\cite{torralba200880} and 50 for ImageNet~\cite{russakovsky2015imagenet}. During this process, each window is hashed using the SHA-256 algorithm. Finally, the resulting set of hashes obtained from all the windows is considered as the ``feature'' for the image. For efficiency purposes, Blacklight utilizes only the top 50 hashes. To quantify the distance between two hash sets, Blacklight computes the number of non-common hashes, which can be interpreted as an $\ell_1$ distance. 

PIHA also operates on input images with the same pixel range. First, it runs a 3x3 low-pass Gaussian filter with standard deviation 1 over the image. Second, the image is converted to the HSV color space with the S and V components discarded. Finally, PIHA runs a sum-pooling operation over 7x7 image blocks, and the ``feature'' is computed as the output of the local binary pattern algorithm~\cite{ojala1994performance} on the sum-pooled image.

\noindent\textbf{Datasets}. We evaluate Blacklight and PIHA using two datasets, TinyImages and ImageNet. The TinyImages dataset is a collection of 32x32 images and is the superset collection from which the popular CIFAR-10 dataset is sampled (providing nearly 80 million images as opposed to only 60,000). The ImageNet dataset comprises over 1 million 256x256 images. We sample a random subset of 1 million images from both datasets for our experiments.

\begin{figure*}[t]
  \centering

  \subfigure{%
    \includegraphics[width=0.25\textwidth]{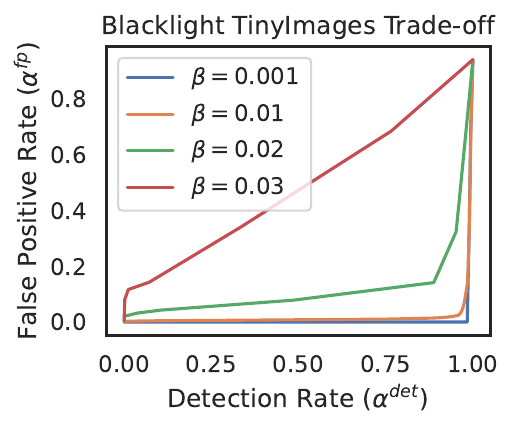}
    \label{fig:subfigure1}
  }%
  \subfigure{%
    \includegraphics[width=0.25\textwidth]{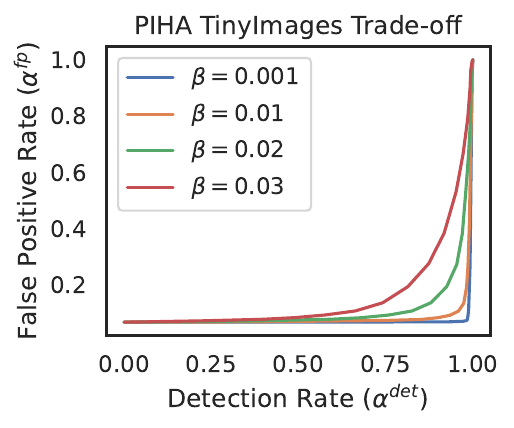}
    \label{fig:subfigure2}
  }%
  \subfigure{%
    \includegraphics[width=0.25\textwidth]{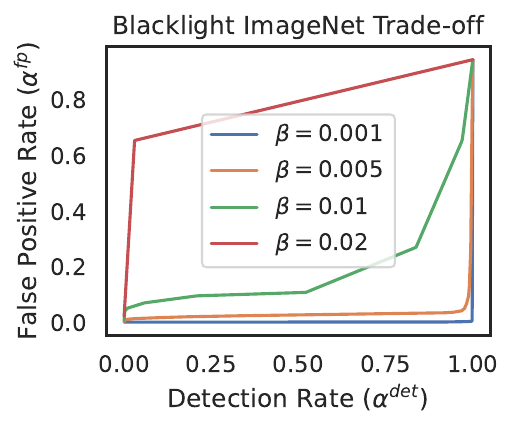}
    \label{fig:subfigure2}
  }%
  \subfigure{%
    \includegraphics[width=0.25\textwidth]{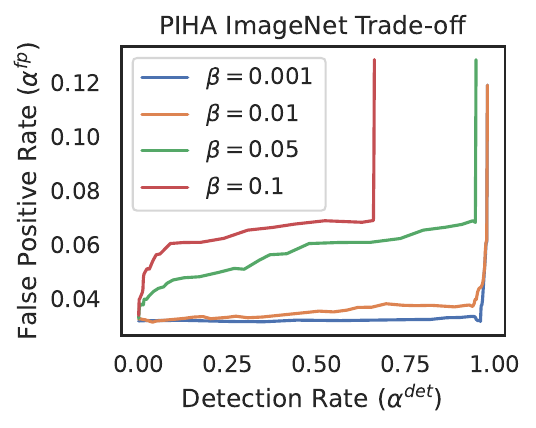}
    \label{fig:subfigure2}
  }%

  \caption{\textit{There exists a trade-off between detection rate $\alpha^{det}$ and false positive rate $\alpha^{fp}$ for stateful defenses. This trade-off is worsened for larger $\beta$ values.} Each curve is computed by varying threshold $\tau$ for the chosen feature extractor, and each setting presents four curves corresponding to different $\beta$ values.}
  \label{fig:beta}
\end{figure*}

\subsection{Q1. Variance of Attack Queries}\label{sec:q1} 
Theorem~\ref{theorem:general} suggests a clear inverse relationship between the ($\alpha^{det}$, $\alpha^{fp}$) trade-off and $\beta$. We now empirically validate this relationship, i.e., for any given feature extractor and dataset, we plot $\alpha^{fp}$ against $\alpha^{det}$ for a variety of thresholds $\tau$. We compute $\alpha^{fp}$ over 1 million images for all settings except PIHA on ImageNet, for which we compute on $100$k and extrapolate due to computational complexity. We compute $\alpha^{det}$ over 100 images by sampling perturbations from Gaussians with different standard deviations $\beta$. 

Results are presented in Figure~\ref{fig:beta}. Notably, we first observe that for any $\beta$, the trade-off between $\alpha^{det}$ and $\alpha^{fp}$ indeed exists across all thresholds.  More specifically, to obtain a larger $\alpha^{det}$ always requires an increase in $\alpha^{fp}$ as well. This validates the takeaways from Theorems~\ref{theorem:toy} and~\ref{theorem:general}. Furthermore, the inverse relationship with $\beta$ also exists, i.e., achieving the same $\alpha^{det}$ requires a larger $\alpha^{fp}$ when $\beta$ is increased. Interestingly, PIHA can achieve higher $\alpha^{det}$ on the low-dimensional TinyImages compared to Blacklight, but both suffer on ImageNet when $\beta$ increases beyond $\beta=0.01$. 

\subsection{Q2. Lipschitz Constants of the Feature Extractor}
Theorem \ref{theorem:general} also suggests that the ($\alpha^{det}$, $\alpha^{fp}$) trade-off is influenced by Lipschitz constants $K_U$ and $K_L$ of the feature extractor. However, this assumes a continuous feature extractor --- although the feature extractors from Blacklight and PIHA are not continuous, they are still designed to approximate the perceptual likeness of images (yielding closer features for similar queries and further features for dissimilar ones). Given the lack of closed-form expressions, we resort to an empirical estimation of $K_U$ and $K_L$.

We create image pairs $\mathbf{x}$ and $\mathbf{x}+\boldsymbol{\delta}$ where $\mathbf{x}$ is sampled from the dataset (TinyImages/ImageNet), and $\boldsymbol{\delta} \sim \mathcal{N}(0,\mathbf{I_d}\beta^2); \beta = 0.01$. For each pair, we then calculate the ratio between the $\ell_2$ distance in the feature space and the input space, i.e., $\frac{||H(\mathbf{x}) - H(\mathbf{x}+\boldsymbol{\delta})||}{||\boldsymbol{\delta}||}$. We construct $10000$ such pairs and plot the distribution of these distance ratios. 

Figure~\ref{fig:lipschitz} plots these distributions for ImageNet images processed by both Blacklight and PIHA feature extractors. We note a larger distribution spread in the histogram for PIHA compared to Blacklight, hinting at a greater value for $\frac{K_U}{K_L}$ for PIHA. As per Theorem~\ref{theorem:general}, this suggests that PIHA possesses the potential for superior detection rates compared to Blacklight. We corroborate this empirically by plotting $\alpha^{fp}$ against $\alpha^{det}$ in a manner akin to that in Section \ref{sec:q1}. As presented in Figure \ref{fig:lip_tradeoff}, PIHA indeed manifests higher detection rates when compared with Blacklight.

\begin{figure}[htbp]
  \centering

  \subfigure{%
    \includegraphics[width=0.23\textwidth]{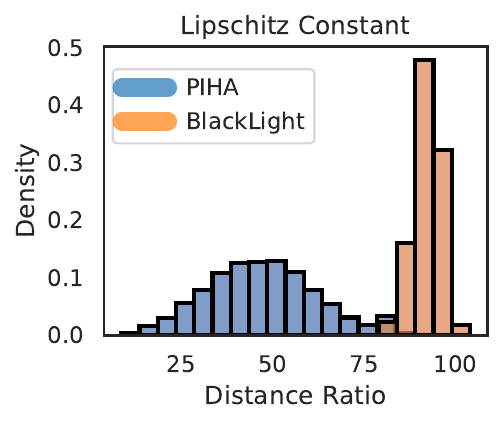}
    \label{fig:lipschitz}
  }%
  \subfigure{%
    \includegraphics[width=0.26\textwidth]{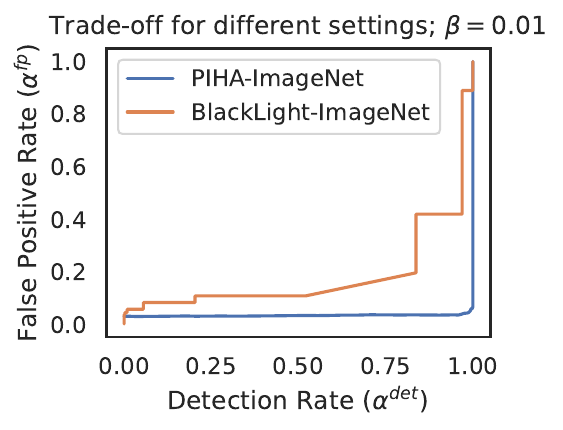}
    \label{fig:lip_tradeoff}
  }%

  \caption{\textit{Lipschitz constant ratio of the feature extractors is directly proportional to the quality of the trade-off.} On the left, we present the distribution of ratios between pairwise distance in the feature space and pairwise distance in the input space --- a larger distribution spread implies a larger Lipschitz ratio for that feature extractor. On the right, we present the corresponding ($\alpha^{det}$, $\alpha^{fp}$) trade-off. }
  \label{fig:subplots}
\end{figure}

\begin{figure*}[htbp]
  \centering

  \subfigure{%
    \includegraphics[width=0.25\textwidth]{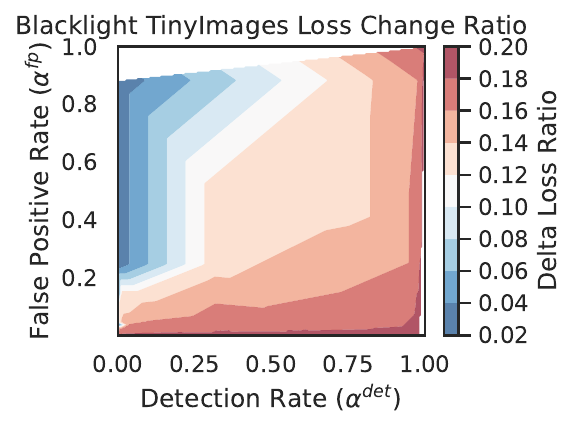}
    \label{fig:subfigure1}
  }%
  \subfigure{%
    \includegraphics[width=0.25\textwidth]{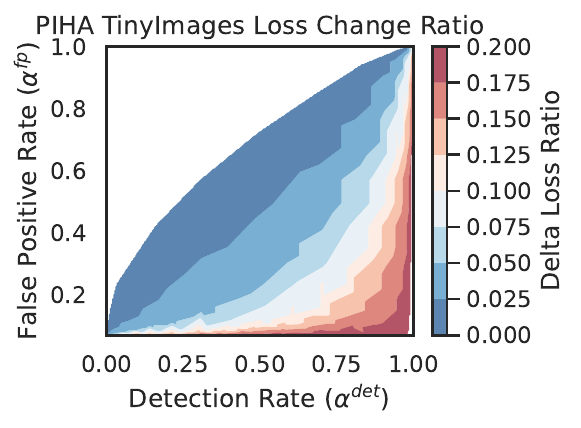}
    \label{fig:subfigure2}
  }%
  \subfigure{%
    \includegraphics[width=0.25\textwidth]{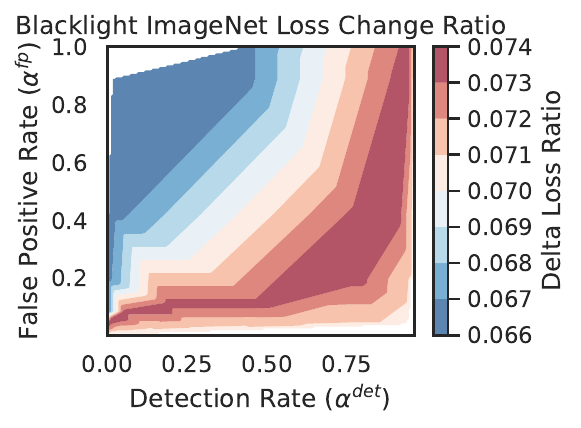}
    \label{fig:subfigure2}
  }%
  \subfigure{%
    \includegraphics[width=0.25\textwidth]{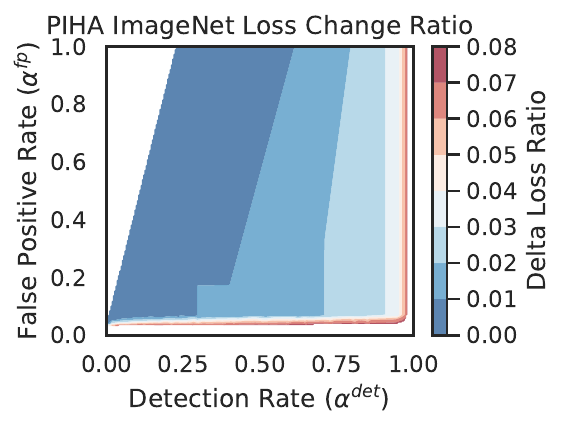}
    \label{fig:subfigure2}
  }%

  \caption{\textit{Even though the ($\alpha^{det}$, $\alpha^{fp}$) trade-off is worsened for larger $\beta$ values, the larger $\beta$ values also produce an inferior direction which does not increase the loss much.} Each shaded region corresponds to the change in loss induced by a gradient estimated with different $\beta$ values. Red regions imply a large increase in loss, and blue imply a small increase.}
  \label{fig:change_of_loss}
\end{figure*}

\subsection{Q2. The Trade-off and Attack Convergence}
Given that increasing $\beta$ worsened the trade-off of the defense (Q1 in Section~\ref{sec:q1}), we now question the impact of increasing $\beta$ on the attack convergence itself. We specifically consider the adversary goal of gradient estimation via finite differences. Formally, it can be shown through the following result that increasing $\beta$ should worsen the quality of the estimated gradient:

\begin{theorem}\label{lemma:min_var_grad} 
Let $\nabla_x$ be the true gradient of $\mathbf{x}$ for the classifier's loss, and $G$ be a matrix of rows $g_1, \cdots, g_k \sim \mathcal{N}(0, \mathbf{I_d}\beta^2)$. Then, the norm of estimated gradient $G\cdot\nabla_\mathbf{x} $ is bounded in probability by:
\begin{equation*}
\begin{split}
    \mathbb{P}[(1-\epsilon)\|\nabla_\mathbf{x}\| \leq \|G\cdot\nabla_\mathbf{x}\| & \leq (1+\epsilon) \|\nabla_\mathbf{x}\|] \geq  \\
     & 1 - 2\cdot exp{\left(- k - \dfrac{1 + \epsilon}{2\beta^2}\right)}
\end{split} 
\end{equation*}
where $0 \leq \epsilon \leq 1$ is the estimation error.
\end{theorem}
A detailed proof of this result can be found in Appendix~\ref{app:proof_min_grad_var}. The left-hand side represents the probability that our estimated gradient is ``good'', i.e., produces the same increase-in-loss as the true gradient. As $\beta$ increases, the lower bound on this probability decreases (right-hand side), suggesting that the estimate is less likely to produce the same increase-in-loss.

We empirically validate this impact of increasing $\beta$ in Figure~\ref{fig:change_of_loss}, which plots the increase in loss when following gradients estimated with different $\beta$. These figures present a clearer overall picture --- for any given $\alpha^{fp}$, even though larger $\beta$ decreases the detection rate, a gradient estimated with larger $\beta$ is also strictly worse for the adversary, i.e., does not increase the loss as much (see the gradation from red to blue). In other words, these findings suggest that the worsening of the ($\alpha^{det}$, $\alpha^{fp}$) trade-off at larger $\beta$ is not without a negative impact on the adversary.

\section{Conclusion}
In conclusion, our work offers a more formal understanding of how stateful defenses prevent black-box adversarial attacks. We outlined a crucial trade-off between detecting attack detection and false positives, and highlighted its dependence upon the distribution of attack and natural queries, and the properties of the defense's feature extractor. Our analysis can help illuminate why certain defenses perform better against black-box attacks, which can help to refine current strategies and potentially guide the design of future defenses. As the landscape of adversarial attacks and defenses evolves, our findings contribute to the development of more robust and resilient machine learning models under the realistic black-box threat model.
\section{Acknowledgements}
This material is based upon work supported by DARPA under agreement number 885000, National Science Foundation Grant No. 2039445, and National Science Foundation Graduate Research Fellowship Grant No. DGE 1841052. Any opinion, findings, and conclusions or recommendations expressed in this material are those of the authors(s) and do not necessarily reflect the views of our research sponsors.

\nocite{langley00}

\bibliography{example_paper}
\bibliographystyle{icml2023}

\newpage
\clearpage
\appendix
\section{Supplementary Proofs}\label{sec:concentration}

\begin{lemma}
Let $G$ be a $k \times d$ random matrix with rows $\sigma g^i \sim \mathcal{N}(0, \mathbf{I_d}\sigma^2) ~\forall 1 \leq i \leq k$. Then, for any unit vector $v \in \mathbb{R}^d$,
\begin{equation*}
    P[|\|Gv\|^2 - 1| > \epsilon] \leq 2 exp\left(-\left(k + \dfrac{\epsilon + 1}{2 \sigma^2}\right)\right)
\end{equation*}
\label{lemma:concn_msur}
\end{lemma}

\begin{proof} Note that by rotational invariance of Gaussians, $Gv \stackrel{D}{=} Ge^1$, where $e^1$ is the standard basis vector. This implies that $\|Gv\|^2 \stackrel{D}{=} \|Ge^1\|^2 \stackrel{D}{=} \sigma^2 \chi_k^2$, where $\chi_k^2$ is a chi-square random variable with $k$-degrees of freedom. Then, by Chernoff's bounding method:

\begin{equation*}
\begin{split}
    P[|\|Gv\|^2 - 1| > & \epsilon] = P[|\sigma^2 \chi_k^2 - 1| > \epsilon] \\ \\
    &\leq 2 \inf\limits_{t>0}e^{-\epsilon t}\mathbb{E}[e^{t(\sigma^2\chi^2-1)}] \\
    &\leq 2 \inf_{t>0} e^{-\epsilon t - t} \mathbb{E}[e^{t\sigma^2 \chi^2}] \\
    & \leq 2 \inf_{t>0} e^{-\epsilon t - t}(1-2\sigma^2 t)^{\frac{-k}{2}} \\
    & \leq 2 \inf_{t > 0} e^{-\epsilon t - t - \frac{k}{2} \log(1-2\sigma^2 t)} \\
    & \leq 2 e^{- \frac{\epsilon - k \sigma^2 + 1}{2\sigma^2} - \frac{k}{2}\log(\frac{\sigma^2 k}{\epsilon + 1})} \\
    & \leq 2 e^{- \frac{(\epsilon + 1 -\sigma^2 k)^2}{2\sigma^2(1+\epsilon)}} \\
    & \leq 2e^{- \frac{(\epsilon+1)^2 - 2\sigma^2 k(\epsilon+1)}{2\sigma^2 (1+\epsilon)}} \\
    & \leq 2e^{- \frac{\epsilon+1 - 2\sigma^2 k}{2\sigma^2}} \\
    & \leq 2e^{-k}e^{-\frac{\epsilon+1}{2\sigma^2}} \\
    & \leq 2e^{-k - \frac{\epsilon+1}{2\sigma^2}}
\end{split}
\end{equation*}
\end{proof}

\subsubsection{Proof for Theorem~\ref{lemma:min_var_grad}}\label{app:proof_min_grad_var}
Let $\nabla_\mathbf{x}$ be the true gradient of $\mathbf{x}$ for the classifier's loss, and $G$ be a matrix of rows $g_1, \cdots, g_k \sim \mathcal{N}(0, \mathbf{I_d}\beta^2)$. Then, the norm of estimated gradient $G\cdot\nabla_\mathbf{x} $ is bounded in probability by:
\begin{equation*}
\begin{split}
    \mathbb{P}[(1-\epsilon)\|\nabla_\mathbf{x}\| \leq \|G\cdot\nabla_\mathbf{x}\| & \leq (1+\epsilon) \|\nabla_\mathbf{x}\|] \geq  \\
     & 1 - 2\cdot exp{\left(- k - \dfrac{1 + \epsilon}{2\beta^2}\right)}
\end{split} 
\end{equation*}
where $0 \leq \epsilon \leq 1$ is the estimation error.

\begin{proof}
    \begin{equation*}
    \begin{split}
   &  P[(1-\epsilon)\|\nabla_\mathbf{x}\| \leq \|G\nabla_\mathbf{x}\| \leq (1+\epsilon)\|\nabla_\mathbf{x}\|] \\
   & = P[(1-\epsilon)^2\|\nabla_\mathbf{x}\|^2 \leq \|G\nabla_\mathbf{x}\|^2 \leq (1+\epsilon)^2\|\nabla_\mathbf{x}\|^2] \\
   & \geq P\left[1-\epsilon \leq \dfrac{\|G \nabla_\mathbf{x}\|^2}{\|\nabla_\mathbf{x}\|^2} \leq 1 + \epsilon \right] \\
   & = P\left[\left|\dfrac{\|G \nabla_\mathbf{x}\|^2}{\|\nabla_\mathbf{x}\|^2} - 1\right| \leq  \epsilon \right] \\
   &= P\left[\left|\left\|G\dfrac{ \nabla_\mathbf{x}}{\|\nabla_\mathbf{x}\|}\right\|^2 - 1\right| \leq  \epsilon \right] \\
   & \geq 1 - 2e^{-k - \frac{\epsilon+1}{2\beta^2}}
    \end{split}
    \end{equation*}
    Where the last step is by Lemma 1.
\end{proof}

\end{document}